\documentclass[conference]{IEEEtran}

\usepackage{amssymb}
\usepackage{amsmath}
\usepackage{graphicx}
\usepackage{amsthm}
\usepackage{tikz}
\usepackage{listings}  
\usepackage[ruled]{algorithm2e}
\usepackage{booktabs} 
\usepackage{subcaption}

\theoremstyle{plain}
\newtheorem{theorem}{Theorem}[section]

\newtheorem{lemma}[theorem]{Lemma}

\newtheorem{theorem-definition}[theorem]{Theorem-Definition}

\theoremstyle{definition}

\newtheorem{assumption}{Assumption}

\theoremstyle{remark}

\newtheorem{remark}[theorem]{Remark}

\renewcommand{\geq}{\geqslant}
\renewcommand{\leq}{\leqslant}

\newcommand{\R}{\mathbb{R}}

\newcommand{\norm}[2]{\left\|{#1}\right\|_{#2}}
\newcommand{\inner}[3]{{\left<{#1},{#2}\right>_{#3}}}

\numberwithin{equation}{section}

\begin{document}

\title{Estimating the Jacobian matrix of an unknown multivariate function from sample values by means of a neural network}

\author{
\IEEEauthorblockN{Fr\'{e}d\'{e}ric Latr\'{e}moli\`{e}re}
\IEEEauthorblockA{Department of Mathematics\\
University of Denver\\
Denver, CO 80208\\
Email: frederic@math.du.edu}
\and
\IEEEauthorblockN{Sadananda Narayanappa}
\IEEEauthorblockA{Lockheed Martin Space\\
Littleton, CO 80127\\
Email: sadananda.narayanappa@lmco.com}
\and
\IEEEauthorblockN{Petr Vojt\v{e}chovsk\'y}
\IEEEauthorblockA{Department of Mathematics\\
University of Denver\\
Denver, CO 80208\\
Email: petr@math.du.edu}
}

\maketitle

\begin{abstract}
We describe, implement and test a novel method for training neural networks to estimate the Jacobian matrix $J$ of an unknown multivariate function $F$. The training set is constructed from finitely many pairs $(x,F(x))$ and it contains no explicit information about $J$. The loss function for backpropagation is based on linear approximations and on a nearest neighbor search in the sample data. We formally establish an upper bound on the uniform norm of the error, in operator norm, between the estimated Jacobian matrix provided by the algorithm and the actual Jacobian matrix, under natural assumptions on the function, on the training set and on the loss of the neural network during training.

The Jacobian matrix of a multivariate function contains a wealth of information about the function and it has numerous applications in science and engineering. The method given here represents a step in moving from black-box approximations of functions by neural networks to approximations that provide some structural information about the function in question.
\end{abstract}


\section{Introduction}

\subsection{The problem}

We propose to use neural networks to approximate the Jacobian matrix $J$ of a multivariate function $F$, where in the training of the neural network we use only a finite sample of input-output pairs $(x,F(x))$. Crucially, no additional information about $F$ or $J$ enters into the training and thus, in essence, the method proposed here differentiates a multivariate function $F$ solely based on a cloud of sample points.

The algorithm is based on two main ideas:
\begin{itemize}
\item a loss function that utilizes a linear approximation of the sampled function $F$ in terms of the sought-after Jacobian matrix,
\item a nearest neighbor search in preparation of the training data from samples.
\end{itemize}
The neural network is not differentiated during or at the conclusion of the training.

In general, to estimate $J$ from sample points of $F$ is a difficult mathematical problem. Since the Jacobian has numerous applications in mathematics, science and engineering, the ability to estimate it by any means, for instance by a neural network, is valuable in its own right.

Of particular importance here is the fact that the Jacobian can be invoked to detect relations, or lack thereof, between the input variables and output values of the unknown function $F$. While neural networks are a powerful tool for nonlinear regression, the resulting interpolating function is notoriously a black-box, revealing little structural information about the sampled function $F$. By taking advantage of the estimated Jacobian matrix of $F$, we can start peering inside $F$. We therefore expect that the ideas presented here will contribute to new approaches for the training of neural networks designed to reveal structural information about sampled data.

\section{Related results}

\subsection{Differentiating trained neural networks that interpolate $F$}

A natural first approach to computing the Jacobian of a function $F$ by a neural network is to train a network to interpolate $F$ as usual and then differentiate the network itself (which, in the end, is a function).

One drawback of this approach is that it is possible to differentiate a neural network only if all its activation functions are differentiable, a situation that complicates the usage of some popular activation functions, such as ReLu \cite{ReLu}. (See \cite{Berner} for a way of handling activation functions that are differentiable almost everywhere.)

A more serious obstacle is the well-known fact that even if a given function is a good approximation of another function---say in the sense of the uniform norm---their derivatives can be drastically different \cite{Rudin}. A simple but illustrative example is given in Figure \ref{uniform-fig}. The sequence of sine waves converges uniformly on $\R$ to the constant zero function since the amplitude of the waves decreases to $0$. However, as the frequencies of the waves grow to infinity, the derivatives of the sine waves grow arbitrarily large and do not converge at all (not even pointwise) over $\R$.

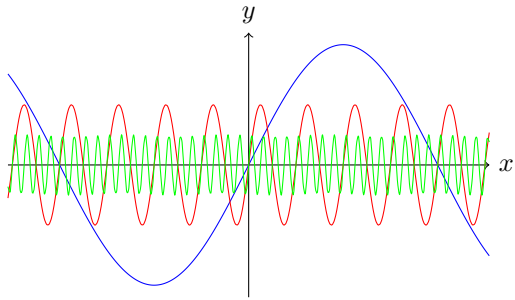
\begin{figure}[!ht]
    \centering
  \begin{tikzpicture}[scale=0.8]
    \draw[->] (-4.0,0.0)--(4.0,0) node[right] {$x$};
    \draw[->] (0,-2.2)--(0,2.2) node[above] {$y$};
    \draw[domain=-4:4, samples=200, smooth, blue] plot ({\x},{2*sin(deg(\x))}) ;
    \draw[domain=-4:4, samples=200, smooth, red] plot ({\x},{sin(8*deg(\x))}) ;
    \draw[domain=-4:4, samples=200, smooth, green] plot ({\x},{0.5*sin(32*deg(\x))}) ;
  \end{tikzpicture}
  \caption{Uniform convergence of functions does not imply converge of their derivatives.}\label{uniform-fig}
\end{figure}

In the context of neural networks, if a neural network is overfitted, then, in a manner similar to Runge's phenomenon \cite{Runge}, the derivative of the overfitted neural network will likely have little resemblance to the derivative of the sampled function $F$.

The difficulty with differentiating a neural network trained to interpolate a function is well demonstrated in \cite{He}, where, among other results, the authors train a neural network $A$ on a sampled function $F$, differentiate the resulting neural network as a function, thus obtaining its Jacobian matrix $J_A$, and then compare the interpolated values of $F$ produced by $A$ with linear approximations of $F$ based on the Jacobian matrix $J_A$. They find that ``the $J$ estimation task fails'' and conclude that neural networks ``may be not suitable to handle derivative signal analysis.''

\subsection{Differentiating neural networks during training}

In the highly influential paper \cite{Lagaris}, the authors proposed a method for training a neural network to solve (higher order) differential equations. The main idea of \cite{Lagaris} is to differentiate the neural network during training and use this (higher order) derivative in the loss function. The training points are typically selected on a regular grid and hence it must be possible to evaluate all relevant functions at such points. The method of \cite{Lagaris} is robust and it can be adopted to multivariate functions. We mention \cite{Rudd} as one of the many papers ultimately based on \cite{Lagaris}.

In the simplest case covered by \cite{Lagaris}, an approximate solution $y$ to the differential equation $y'(x) = F(x)$ is obtained by comparing the neural network derivative $A'$ to $F$ during training. Once $A$ is fully trained, $A'$ should be a good approximation of $F$, that is, $A$ itself should be a good approximation of an antiderivate of $F$.

In order to obtain an approximation for the derivative of $F$, one could now differentiate the trained network $A$ twice (running into the issues described in the previous subsection), or start over with the differential equation $y'(x)=F''(x)$ (which presumably requires $F'$ to be already known), or consider a variation of \cite{Lagaris} in which the neural network $A$ is integrated rather than differentiated during training. Little seems to be known about integration of neural networks.

\subsection{Situations in which the Jacobian matrix can be computed by standard methods}

One might work with data sampled from a function that satisfies a known differential equation, say for physical reasons. In the simplest case where the equation is first order and linear, the Jacobian matrix of the function can thus be computed directly from the differential equation. This essentially avoids the problem of estimating the Jacobian matrix by neural network altogether and relies on standard mathematical methods.

Similarly, the well-known automatic differentiation tool Autograd \cite{Autograd} (that is often used in connection with neural networks) conveniently differentiates a function whose code is written in NumPy \cite{NumPy}, but it requires the function to be given explicitly (in NumPy code), not just by sample values.

\subsection{Gradient estimation}

Gradient estimation is a vast topic, cf. \cite{Fu}. In the context of neural networks, gradient estimation typically refers to an estimation of the gradient of the explicitly given loss function. Gradient estimation is critical for the stochastic gradient descent algorithm, which is in turn key in the learning algorithm of neural networks \cite{Rumelhart}. The backpropagation algorithm \cite{Bryson,Dreyfus,Kelley} is, in essence, an algorithm that efficiently computes the gradient of a loss function using a graph-directed implementation of the chain rule. However, backpropagation relies on symbolic differentiation or on a tool similar to Autograd, as well as on a specific form of the neural network.

Another meaning of the phrase ``gradient estimation'' arises in situations when differentiating a given function is slow or impossible and its Jacobian matrix is therefore merely approximated. For instance, in \cite{Wang} the authors give an efficient approximation of the Jacobian matrix by fast Fourier transform in the context of MRI. Such methods typically rely on a prescribed model for the function whose gradient is being estimated.

\medskip

In contrast with all these results, we do not differentiate a neural network that has been trained to interpolate $F$, nor do we differentiate the neural network during training. Rather, we \emph{directly train a neural network to estimate the Jacobian matrix of an unknown function $F$ from given sample values of $F$}, regardless of the context. We prove rigorously that the Jacobian estimator approaches the Jacobian of $F$ in norm, under reasonable assumptions on $F$, the sample set and the performance of the neural network used.

\section{Preliminaries}

In this section we define the Jacobian matrix, recall the linear approximation formula for multivariate functions, and overview the general interpolation problem with a view toward neural networks. Readers familiar with these topics can skip forward to Section \ref{Sc:Algorithm}.

\subsection{Jacobian matrix and linear approximations}

Suppose that $c$, $d$ are positive integers, $U$ is an open, bounded subset of $\R^d$, and $F : U \to \R^c$ is a Fr{\'e}chet differentiable function on $U$ \cite[Ch. 9]{Rudin}. (Note that $d$ stands for the dimension of the domain $\R^d$ and $c$ for the dimension of the codomain $\R^c$.)

Writing $x=(x_1,\ldots,x_d)$ for $x\in\R^d$ and $F = (F_1,\ldots,F_c)$ with $F_j : \R^d \rightarrow \R$ for each $j\in \{1,\ldots,c\}$, the \emph{Jacobian matrix} of $F$ is the matrix
\begin{equation*}
  J = J_F = \left(
  \begin{array}{ccc}
    \frac{\partial F_1}{\partial x_1} & \cdots & \frac{\partial F_1}{\partial x_d} \\
    \vdots & & \vdots \\
    \frac{\partial F_c}{\partial x_1} & \cdots & \frac{\partial F_c}{\partial x_d}
  \end{array}\right)\text.
\end{equation*}
In words, $J_F$ is the matrix whose rows are the gradients of $F_1,\dots,F_c$.

Since the Jacobian $J_F$ is the matrix of the Fr{\'e}chet derivative of $F$ with respect to the canonical basis, it satisfies by definition the following general linear approximation property (which is of course the very idea of the derivative):
\begin{equation}\label{J-eq}
  \lim_{h\rightarrow 0} \frac{F(x+h)-F(x)-J_F(x)\cdot h }{\norm{h}{\R^d}} = 0\text,
\end{equation}
where $\norm{h}{\R^d}$ is the usual Euclidean norm
\begin{displaymath}
    \norm{h}{\R^d} = \norm{(h_1,\dots,h_d)}{\R^d} = \left(\sum_{i=1}^d h_i^2\right)^{1/2}.
\end{displaymath}
Equation \eqref{J-eq} will play a crucial role in the loss function of our neural network.

\subsection{The interpolation problem in general}

In a regression problem with noise, we are given a finite set $S = \left\{ (x,F(x)+\epsilon_x) : x \in X \right\}$ of pairs of input-output values sampled from an unknown function $F : U\subseteq\R^d \rightarrow \R^c$, where $X\subseteq U$ and $(\epsilon_x)_{x\in X}$ is a family of independent random variables, all with mean $0$, representing noise. The goal is to find an estimate for $F$.

If we make the very strong assumption that $F$ is linear, i.e., that there exists a $c\times d$ matrix $A \in M_{c\times d}(\R)$ such that $F(x)=Ax$ for all $x \in U$, then we may apply the standard linear, least square statistical method to derive an estimate for $A$, by minimizing the error
\begin{equation*}
  B \in M_{c\times d}(\R) \mapsto \min_{x \in X} \norm{Bx - F(x)}{\R^c}^2.
\end{equation*}
In this case, we note that the Jacobian matrix of $F$ is again $A$ and, of course, $A$ contains a lot of useful information about $F$.

Without the assumption of linearity of $F$, the problem becomes significantly more complicated and computationally intensive. A general method is to replace the algebra $M_{c\times d}(\R)$ of matrices by a set of nonlinear functions which is parametrized by some points in the parameter space $\R^P$, where $P$ is typically quite large. The main example of interest here is a class of functions called \emph{artificial neural networks}, which, in their simplest form, can be described as a finite chain of alternating compositions of linear functions and functions called activation functions.

Formally, an \emph{artificial neural network} $A$ with $\ell$ layers, with $n_0$ inputs, and with $n_j$ neurons and activation function $f_j:\R\rightarrow\R$ in layer $j$ for each $j\in \{1,\ldots,\ell\}$, is the function
\begin{equation*}
  A = f_\ell^{\times n_\ell} \circ W_\ell \circ f_{\ell-1}^{\times n_{\ell-1}} \circ W_{\ell-1} \circ \ldots \circ f_1^{\times n_1} \circ W_1,
\end{equation*}
where, for each $j \in \{1,\ldots,\ell\}$, the matrix of weight $W_j$ (identified above with the linear map $x\mapsto W_jx$) has size $n_j\times n_{j-1}$, and $f_j^{\times n_j}$ maps $(x_1,\ldots,x_{n_j}) \in\R^{n_j}$ to $(f_j(x_1),\ldots,f_j(x_{n_j}))$. A typical point of view is that the weight matrices $W_j$ are the parameters of the neural network, while the number of layers, the number of neurons in individual layers and the activation functions are fixed for the given problem.

Once the class of neural networks is fixed, the regression problem is then to find a neural network $A$ as above which is a good approximation to a solution of the minimization problem
\begin{equation*}
    B\text{ neural network} \mapsto \min_{x\in X} \norm{B(x) - F(x)}{\R^c}^2 \text.
\end{equation*}

Solving this minimization problem is in general very difficult. A key observation is that an algorithm based on the gradient descent (or its variants, especially the stochastic gradient descent) called \emph{backpropagation} has proven effective in practice. The process of (numerically, approximately) solving the minimization problem is referred to as the \emph{training} of the neural network.

Under advantageous conditions, we can then interpolate between the values given in the sample set $S$ by means of a trained neural network $A$, thus obtaining a (hopefully good) approximation to $F$ on its domain $U$. However, unlike in the linear case, the obtained neural network $A$ is a \emph{black-box} estimator in that it reveals no particular structure of the function $F$. It is therefore nontrivial to use $A$ for anything more than interpolating the sample data.

\subsection{The difficulty in estimating the Jacobian matrix directly from sample points}

Given a set of sample points $\{(x,F(x)) : x\in X \}$ with $X\subseteq U$ finite (and well-distributed over $U$, for instance, $\varepsilon$-dense for an $\varepsilon$ that is small enough for the scale of the problem), estimating the partial derivatives of $F$ is a difficult problem. Here are some reasons why:
\begin{itemize}
\item The sample set $X$ is random and thus, given $x=(x_1,\ldots,x_d) \in X$, we cannot assume that it contains points of the form $(x_1+\Delta x_1,x_2,\ldots,x_d)$, etc, aligned with $x$ in one of the cardinal directions. Consequently, estimating partial derivatives involves a change of basis \emph{at each point}, incurring a lot of computations and numerical errors.
\item Estimating a partial derivative is difficult in general. One way to see the statistical issue is as follows. If $x \in X$ and $x+h \in X$ then $\frac{F(x+h)-F(x)}{h}$ can be used as an approximation for $J(x)$, but any error on $F(x)$ and $F(x+h)$, even if the error is small, is amplified upon dividing by the small quantity $h$. More formally, since we really only know $F(x)+\epsilon_x$ and $F(x+h)+\epsilon_{x+h}$, then $\frac{\epsilon_{x+h}-\epsilon_x}{h}$ will typically be large for small $h$. If the noise variables $\epsilon_x$ and $\epsilon_{x+h}$ have variance $\sigma^2$, then the variance of $\frac{\epsilon_{x+h}-\epsilon_x}{h}$ is $\frac{2\sigma^2}{h}$, which is large when $h$ is small.
\item Even if we could surmount the above two issues, we would still need some regression technique, such as neural networks, to interpolate the values of partial derivatives at inputs not contained in $X$.
\end{itemize}

\section{The algorithm}\label{Sc:Algorithm}

In this section we show in detail how to move one step beyond the regression problem for functions and present an algorithm that estimates the Jacobian matrix of a sampled function $F$ by means of a neural network. The only assumption on $F:U\to\R^d$ is that it is differentiable on the bounded open set $U\subseteq\R^d$. The main idea is to use equation \eqref{J-eq} to estimate $J$ directly from pairs of points in the data cloud $\{(x,F(x)):x\in X\subseteq U\}$.

\subsection{Illustrating the algorithm}

Let us first illustrate the main idea of the algorithm in the simplest case $c=d=1$. Suppose that we wish to train a neural network $\widehat J$ for approximating the derivative of $F:\R\to\R$ from a finite sample set $\{(x,F(x)):x\in X\}$. Suppose that in the process of training $\widehat J$ we encounter a sample point $(a,F(a))$. Let $(b,F(b))$ be another sample point. If $b$ is close enough to $a$ then $F(b)$ is approximately equal to
\begin{displaymath}
    F(a) + F'(a)(b-a).
\end{displaymath}
Since $\widehat J$ is being trained to approximate the unknown derivative $F'$, we naturally consider the known quantity
\begin{displaymath}
    F(a) + \widehat J(a)(b-a)
\end{displaymath}
instead and compare it to $F(b)$, see Figure \ref{Fig1D}.

\begin{figure}[!ht]
\begin{tikzpicture}[scale=1.72]
\draw[domain=1/2:3.8, smooth, variable=\x, blue, thick] plot ({\x}, {3-(\x-3)*(\x-3)/4)});
\draw[domain=1/2:3.8, smooth, variable=\x, red] plot ({\x}, {\x/2+7/4});
\draw[domain=1/2:3.8, smooth, variable=\x, green] plot ({\x}, {\x+3/4});
\draw[dotted] (3,1.9) -- (3,19/4-0.2);
\draw[dotted] (2,1.9) -- (2,19/4-0.2);
\draw[dotted] (5/4-0.1,3) -- (3.8,3);
\draw[dotted] (5/4-0.1,15/4) -- (3.8,15/4);
\draw[dotted] (5/4-0.1,13/4) -- (3.8,13/4);
\draw[dotted] (5/4-0.1,11/4) -- (3.8,11/4);
\draw[->] (1/2,2) -- (3.8, 2) node[right] {$x$};
\draw[->] (5/4, 5/4) -- (5/4, 19/4-0.2) node[above] {$y$};
\draw[black,fill=black] (2,11/4) circle (.18ex);
\draw[black,fill=black] (3,3) circle (.18ex);
\draw[black,fill=black] (3,13/4) circle (.18ex);
\draw[black,fill=black] (3,15/4) circle (.18ex);
\node at (0.9,11/4) {$F(a)$};
\node at (0.9,3) {$F(b)$};
\node at (0.31,13/4) {$F(a){+}F'(a)(b{-}a)$};
\node at (0.31,15/4) {$F(a){+}\widehat J(a)(b{-}a)$};
\node at (2,1.8) {$a$};
\node at (3,1.8) {$b$};
\node at (3.9,11/4+0.1) {$F$};
\draw[|-|] (4.1,3)-- node[right] {loss} (4.1,15/4) ;
\end{tikzpicture}
\caption{The loss (before normalization) for the Jacobian matrix estimator $\widehat J$ resulting from the sample points $(a,F(a))$, $(b,F(b))$ of an unknown function $F:\R\to\R$. The secant line through $(a,F(a))$ with slope $\widehat J(a)$ is in green. The unknown tangent line through $(a,F(a))$ is in red.}\label{Fig1D}
\end{figure}
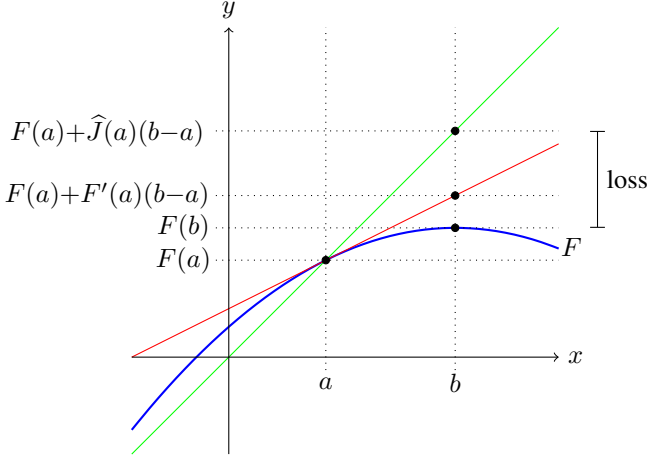

We use the related normalized quantity
\begin{equation}\label{Loss1D}
    \frac{|F(b)-(F(a)+\widehat J(a)(b-a))|^2}{|b-a|^2}.
\end{equation}
as a contribution to the loss function.

In the general case of estimating the Jacobian matrix of $F:\R^d\to \R^c$, the only differences from the $\R\to\R$ case are:
\begin{itemize}
\item The neural network $\widehat J$ takes a point in $\R^d$ as an input and returns a real $c\times d$ matrix.
\item While training $\widehat J$ at $a\in X$, we do not use a randomly chosen sample point $b\in X$ close to $b$ but rather the $k_{max}$ nearest neighbors of $a$ within radius $r_{max}$, where $k_{max}$ and $r_{max}$ are parameters of the algorithm. The nearest neighbors are precalculated and processed in batches.
\item The loss function \eqref{Loss1D} is replaced with its multivariate analog, that is, for every data point $(a,F(a))$ and its neighbor $(b,F(b))$, the contribution to the loss is
\begin{equation}\label{Loss}
    \frac{\norm{F(b)-(F(a)+\widehat J(a)(b-a))}{\R^c}^2}{\norm{b-a}{\R^d}^2}\text.
\end{equation}
\end{itemize}

\subsection{The algorithm}

A pseudo-code for the Jacobian matrix estimator can be found in Algorithm 1.

\begin{algorithm}[!ht]
\DontPrintSemicolon
\SetKwInOut{Input}{input}\SetKwInOut{Output}{output}
\SetKwFunction{NearestNeighbors}{NearestNeighbors}
\SetKwFunction{Union}{Union}
\caption{JacobianEstimator}\label{J-alg}
\Input{\;
    \Indp
    $X$, an array of $N$ vectors in $\R^d$ \;
    $Y$, an array of $N$ vectors in $\R^c$ \;
    $k_{max}$, a positive integer\;
    $r_{max}$, a positive real number
}
\Output{\;
    \Indp
    $\widehat{J}$, a trained neural network
}
\BlankLine
\tcc{initialize neural network}
$\widehat{J} \leftarrow \text{initial neural network}$\;
\BlankLine
\tcc{data for nearest neighbors}
$D \leftarrow [\ ]$\;
\For{$i \leftarrow 0$ \KwTo $N-1$}{
    $N_i \leftarrow $ \NearestNeighbors{$X,i,k_{max},r_{max}$}\;
    $D[i] \leftarrow \{(i,j):j\in N_i\}$\;
}
$D \leftarrow \Union{D}$ \tcp*{list of pairs}
\BlankLine
\tcc{loss function for a batch $B{\subseteq}D$}
\SetKwBlock{FnLoss}{$f \leftarrow$ function( $B$ )}{}
\FnLoss{
    $s \leftarrow 0$\;
    \For{$d\in B$}{
        $i \leftarrow d[0]$\;
        $j \leftarrow d[1]$\;
        $s \leftarrow s + \frac{\norm{Y[j]-(Y[i]+\widehat{J}(X[i])(X[j]-X[i]))}{\R^c}^2}{\norm{X[j]-X[i]}{\R^d}^2}$\;
    }
    \Return $\frac{s}{|B|}$\;
}\textbf{end}\;
\BlankLine
\tcc{the training cycle}
$\widehat{J} \leftarrow$ train $\widehat{J}$ on $D$ using loss function $f$\;
\Return $\widehat{J}$\;
\end{algorithm}

The algorithm uses standard methods of stochastic gradient descent with backpropagation but it relies on a novel loss function and data preparation which utilizes a nearest neighbor search. An implementation of the algorithm in \textsf{TensorFlow} 2.5.0 \cite{TensorFlow} can be found in the appendix.

\begin{table*}[!ht]
\centering
\caption{The testing functions.}\label{Tb:Functions}
\begin{tabular}{lll}
\toprule
name    &function  &domain\\
  \midrule
  $F_0:\R^2\to\R$     &$(x,y)\mapsto x\exp(-x^2-y^2)$ &$(-2,2)^2$\\
  $F_1:\R^2\to\R$     &$(x,y)\mapsto xy$  & $(-1,1)^2$  \\
  $F_2:\R^2\to\R$     &$(x,y)\mapsto x^3 + 2 x y^2$ &$(0,2)^2$\\
  $F_3:\R^2\to\R$     & $(x,y)\mapsto \ln(1+x^2 y)$ & $(0,2)^2$ \\
  $F_4:\R^2\to\R$     & $(x,y)\mapsto (x+y)(x^2+x y^2 + 1)^{-1}$ & $(-1,1)^2$ \\
  $F_5:\R^2\to\R$     &$(x,y)\mapsto \cos(x^2)+\cos(y^2) + 3x$ &$(-1,1)^2$\\
  $F_6:\R^2\to\R$     &$(x,y)\mapsto \sqrt{1+x}+x\sqrt{1+y}$ & $(0,3)^2$ \\
  $F_7:\R^2\to\R$     &$(x,y)\mapsto \arctan(x+y^2)$ &$(-3,3)^2$\\
  \midrule
  $F_8:\R^3\to\R^2$ & $(x,y,z)\mapsto (x(x+y)+y^2+zx, xyz)$ & $(-1,1)^3$ \\
  $F_9:\R^3\to\R^3$ & $(x,y,z)\mapsto (\sin(x y)+\sin(z y),\cos(x+y)+\cos(x+z),x+y+z)$ & $ (-1,1)^3$ \\
  $F_{10}:\R^4\to\R^4$ & $(x,y,z,t)\mapsto (\sin(x y), \cos(x z)+\cos(y t),\frac{1}{10}(x+y))$ & $(-1,1)^4$ \\
  $F_{11} : \R^5\to\R^2$ & $(x,y,z,t,w)\mapsto (x(z+t)+y w, (x+y) \exp(-z^2-w^2-t))$ & $(-1,1)^5$ \\
  \midrule
  $F_{12}: \R^5\to\R$ & $(x,y,z,t,w)\mapsto \exp(-x^2-\frac{xy}{2}-\frac{3z^2}{2}-t+w)$ & $(-1,1)^5$ \\
  \bottomrule
\end{tabular}
\end{table*}

The sample set is represented as two lists $X\subseteq\R^d$ and $Y\subseteq\R^c$ of the same length $N$. For every index $j$, we think of the pair $(X[j],Y[j])$ as a sample point $(x,F(x))$, where $F$ is some unknown differentiable function $F:U\to\R^c$, $U\subseteq\R^d$.

In Algorithm \ref{J-alg}, the initial neural network $\widehat J$ has $d$ cells in its input layer and $c\times d$ cells in its output layer. The function
\begin{displaymath}
    \textsf{NearestNeighbors}(X,i,k_{max},r_{max})
\end{displaymath}
returns all indices $0\le j<N$ such that $X[j]\ne X[i]$ is among the $k_{max}$ nearest neighbors of $X[i]$ in $X$ and $\norm{X[i]-X[j]}{\R_d}<r_{max}$.

\subsection{Notes on the algorithm}

Let us point out additional features of Algorithm \ref{J-alg} that were used in our implementation and that are not necessarily captured in the pseudo-code.
\begin{itemize}
\item In order not to include many nearby points in one batch, we randomly permute the entries of $D$ at the end of the data preparation stage.
\item Even if $|X|=N=|Y|$ and the parameters $k_{max}$ and $r_{max}$ are fixed, then number of training points in each run of the algorithm depends on $X$ (since we do not know in advance how many nearest neighbors of a given point will satisfy the constraint on $r_{max}$). After the training data $D$ is prepared, we adjust the batch size so that it divides the size of $D$ evenly. Note that the training set therefore has cardinality $|D|$ satisfying $N\le |D|\le Nk_{max}$.
\item We allow to set the value of $r_{max}$ to infinity so that every nearest neighbor set has cardinality $k_{max}$ (if $k_{max}<N$).
\item The effect of the parameter $k_{max}$ can be suppressed by setting its value to at least $N$.
\end{itemize}

Note that it is possible for both $(i,j)$ and $(j,i)$ to occur in the training set $D$. If $(i,j)\in D$, which means that $X[j]$ is one of the $k_{max}$ nearest neighbors of $X[i]$ and within radius $r_{max}$ of $X[i]$, it does not necessarily follow that $(j,i)\in D$, too. With each $(i,j)\in D$, the neural network $\widehat J$ will be trained using a linear approximation centered at $X[i]$.

\subsection{Using the trained Jacobian neural network}

Having trained the neural network $\widehat J$ by Algorithm \ref{J-alg}, we can obtain an approximate value of the Jacobian matrix $J$ of the unknown function $F:U\to\R^c$ by computing $\widehat J(x)$ for $x\in U$.

Note that we can also obtain an approximate value of the unknown function $F$ itself at $x\in U$ by locating a nearby point $y$ in the sample set and computing $F(y) + \widehat{J}(y)(x-y)$.

\section{Examples and results}

We now test and validate Algorithm \ref{J-alg}. The testing functions, their names and their domains are summarized in Table \ref{Tb:Functions}, while the results (error estimates) can be found in Table \ref{Tb:Results}.

\begin{remark}
Most functions in Table \ref{Tb:Functions} are scalar valued (as opposed to vector valued) but this is at little loss of generality. Indeed, a vector valued function $F:\R^d\to\R^c$ can be represented by $c$ scalar valued functions $F_i:\R^d\to\R$, $1\le i\le c$, a neural network can be trained on the same sample set to produce the estimated Jacobian matrix (gradient) $\widehat J_i$ for $F_i$, and the estimated Jacobian matrix $\widehat J$ for $F$ is then obtained as $\widehat J = (\widehat J_1,\dots,\widehat J_c)^\top$. However, since training a neural network on a vector valued function is not the same as training several neural networks on scalar valued functions, we have included a few vector valued functions in Table \ref{Tb:Functions} to demonstrate that Algorithm \ref{J-alg} can cope with that situation as well.
\end{remark}

In Subsection \ref{Ss:Big} we train $\widehat J$ for the function $F_0$ of Table \ref{Tb:Functions} on a sample set consisting of $N=10^6$ points (resulting in a training set whose size is between $N$ and $Nk_{max}$), and we compare $\widehat J$ and $J$ visually as vector fields on a regular grid.

In Subsection \ref{Ss:All} we describe in detail two validation methods, one for the situation when the function $F$ is not known to the training algorithm but is known to us, and another for the situation when the function $F$ is truly unknown. We then train Algorithm \ref{J-alg} on every function from Table \ref{Tb:Functions} using random sample sets of various sizes. As expected, we observe that the error improves with the size of the training set and that it gets worse as the volume of the domain increases.

\begin{figure*}[!ht]
   \begin{subfigure}{\columnwidth}
        \centering
        \includegraphics[width=0.775\columnwidth]{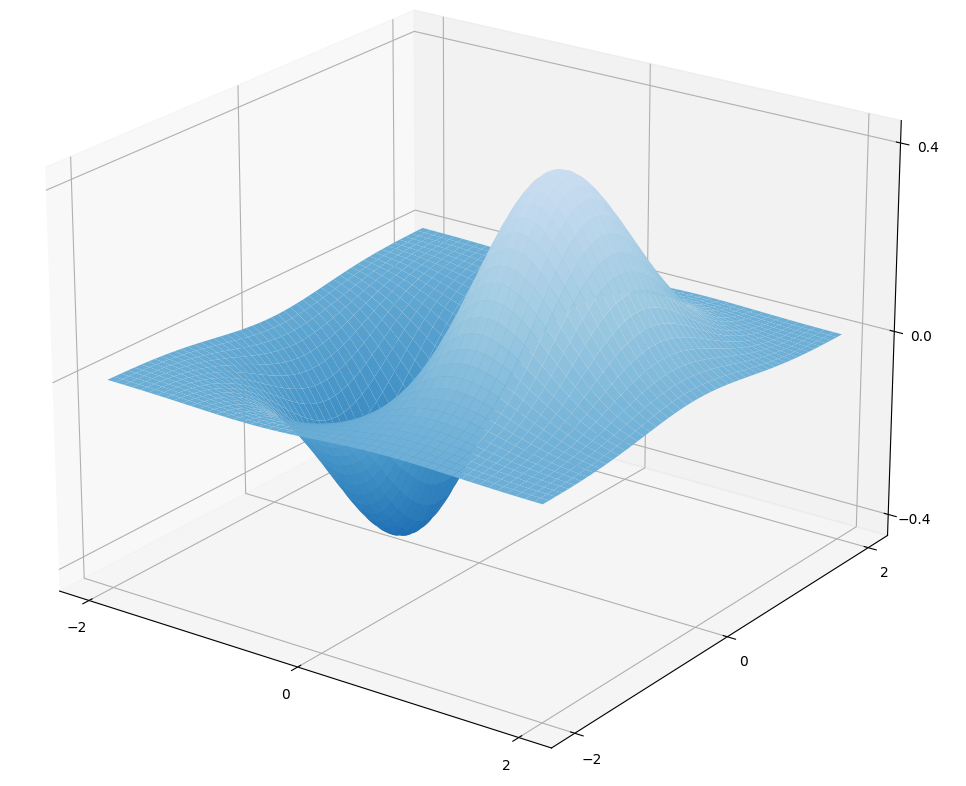}
        \caption{The graph of $F_0$.}
   \end{subfigure}
   \begin{subfigure}{\columnwidth}
        \centering
        \includegraphics[width=0.775\columnwidth, keepaspectratio]{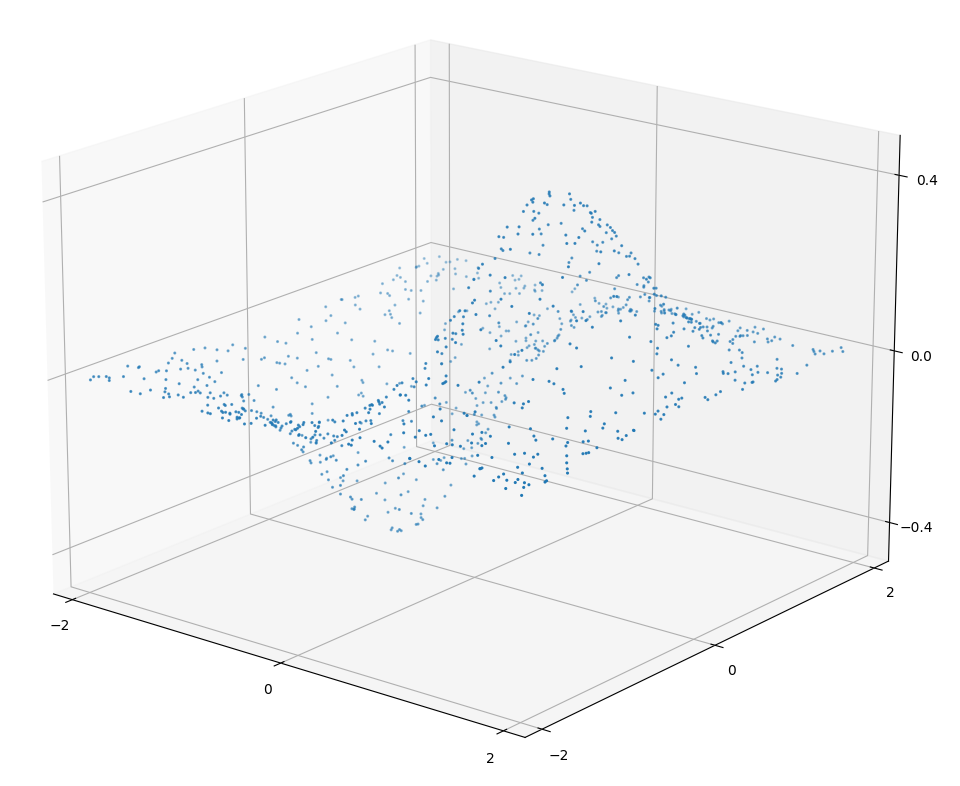}
        \caption{The cloud of training points for $\widehat J$.}
   \end{subfigure}
   \par\vspace*{5mm}
   \begin{subfigure}{\columnwidth}
        \centering
        \includegraphics[width=0.775\columnwidth, keepaspectratio]{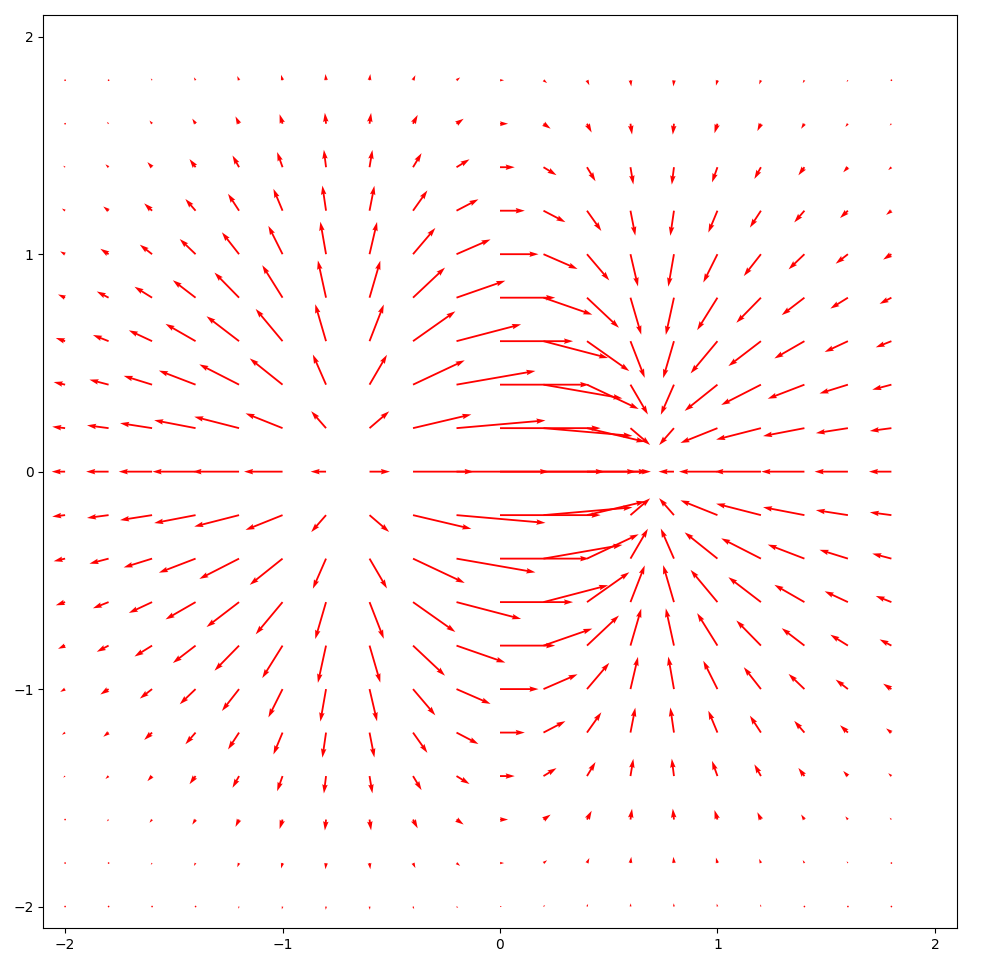}
        \caption{The gradient $J$ (scaled).}
    \end{subfigure}
   \begin{subfigure}{\columnwidth}
        \centering
        \includegraphics[width=0.775\columnwidth, keepaspectratio]{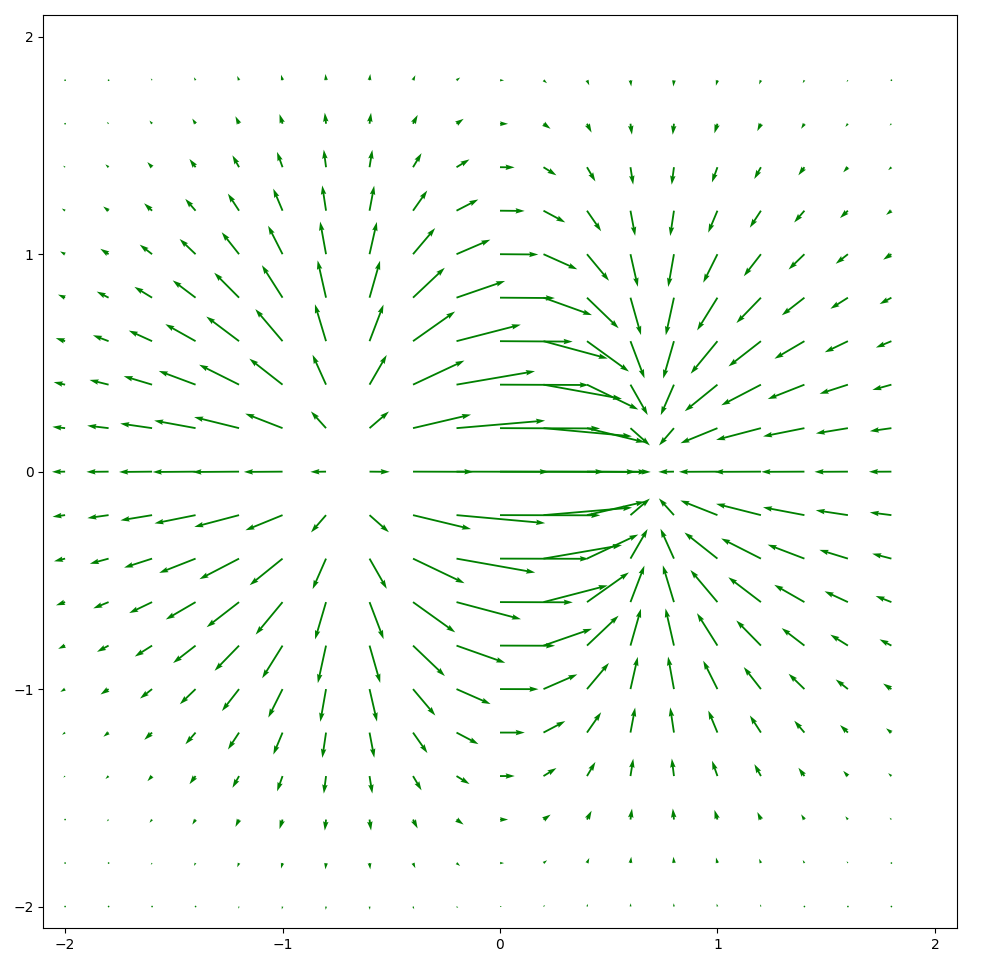}
        \caption{The estimated gradient $\widehat{J}$ (scaled).}
   \end{subfigure}
   \par\vspace*{5mm}
   \begin{subfigure}{\columnwidth}
        \centering
        \includegraphics[width=0.775\columnwidth, keepaspectratio]{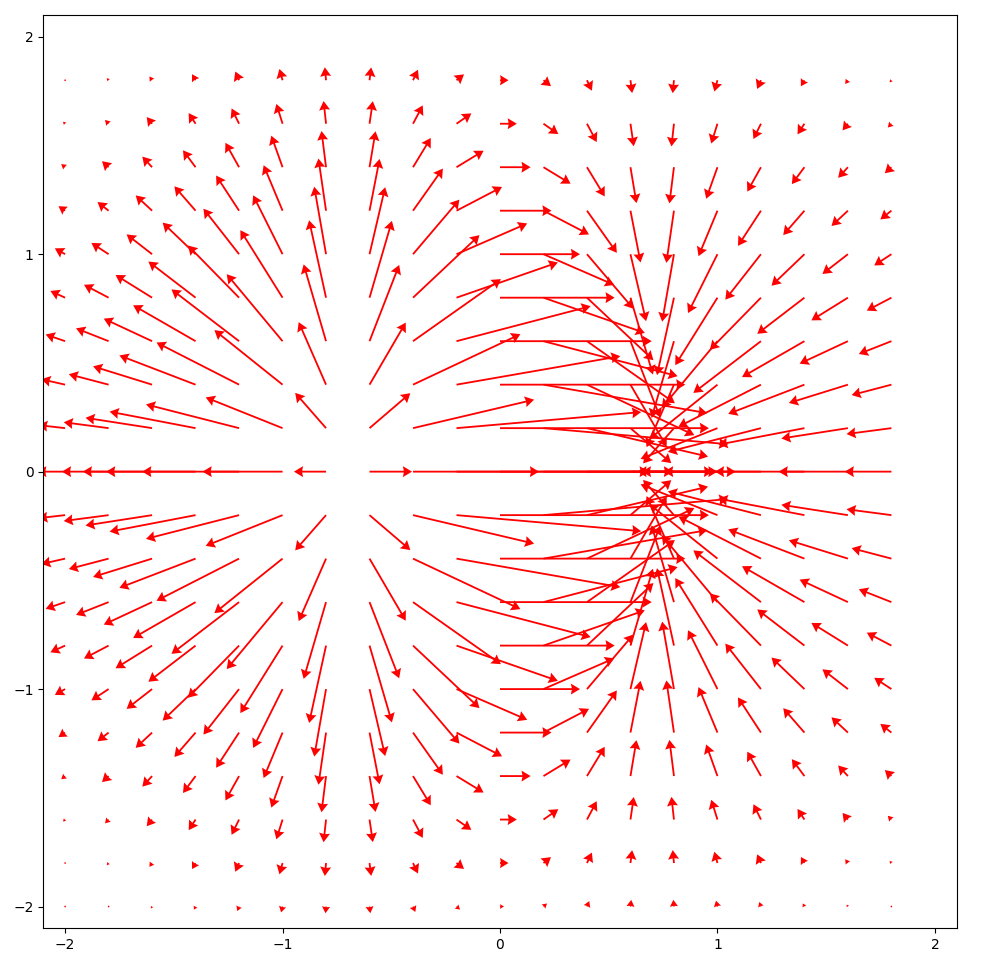}
        \caption{The gradient $J$ (to scale).}
   \end{subfigure}
   \begin{subfigure}{\columnwidth}
        \centering
        \includegraphics[width=0.775\columnwidth, keepaspectratio]{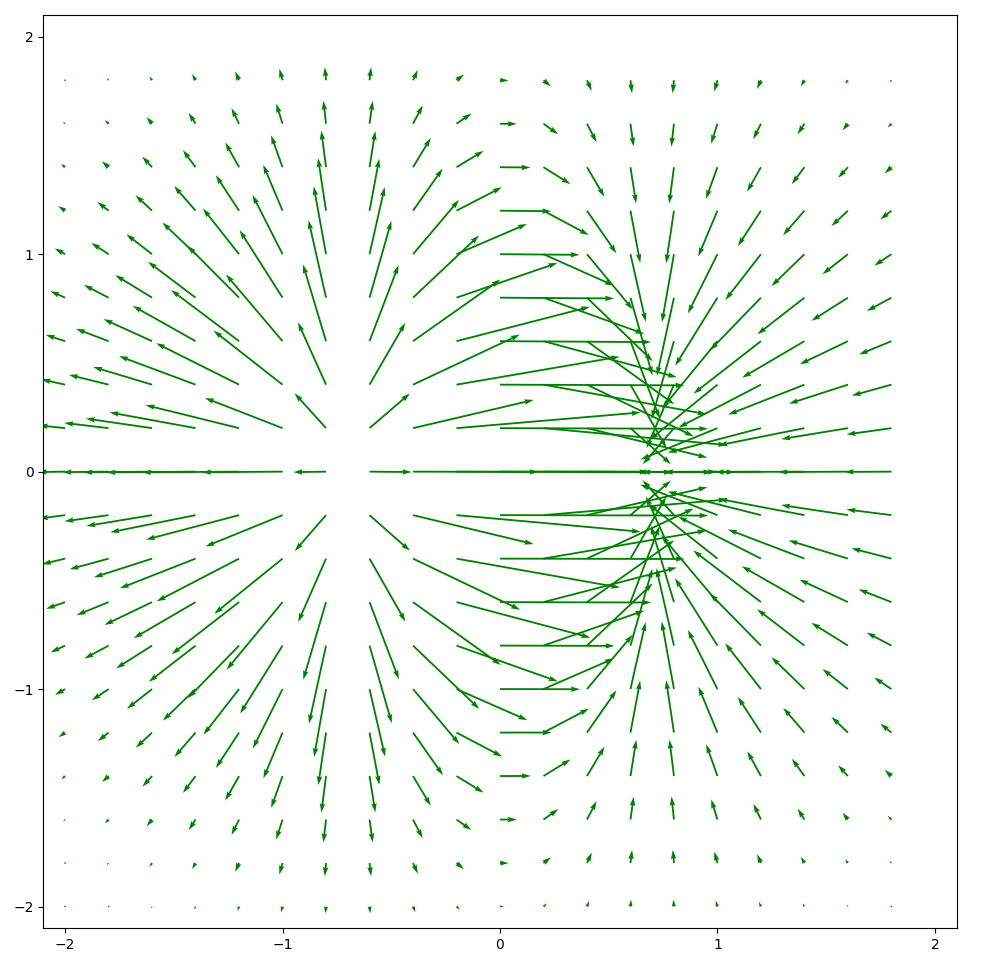}
        \caption{The estimated gradient $\widehat{J}$ (to scale).}
   \end{subfigure}
   \caption{A visual comparison of the Jacobian $J$ of $F_0$ and the trained Jacobian estimator $\widehat J$.}\label{big}
\end{figure*}

Finally, in Subsection \ref{Ss:Notes} we discuss the effects of varying certain parameters of the algorithm (namely $r_{max}$, $k_{max}$, and the geometry of the neural network), as well as the effect of adding noise to the sampling data. We also comment on the limitations of Algorithm \ref{J-alg} and on computing time.

\subsection{A visual example}\label{Ss:Big}

In all examples of Subsections \ref{Ss:Big} and \ref{Ss:All} we set the parameters of the algorithm as
\begin{itemize}
\item $r_{max}=0.5$,
\item $k_{max}=30$,
\end{itemize}
and of the neural network as
\begin{itemize}
\item a simple forward ANN,
\item with $4$ hidden layers consisting of $100$, $100$, $50$ and $20$ neurons,
\item using the swish activation function for all cells,
\item and trained with stochastic gradient descent in $50$ epochs and batch size of $50$.
\end{itemize}

Consider the function $F_0$ of Table \ref{Tb:Functions}. This function is not known to Algorithm \ref{J-alg} for the purposes of training the Jacobian matrix estimator $\widehat J$, nevertheless it is known to us and we can therefore visualize it, calculate its Jacobian matrix $J$ by standard symbolic differentiation, and visualize the Jacobian matrix as well.

The graph of $F_0$ is given in Figure \ref{big}(a). Its Jacobian matrix $J$ is a function $\R^2\to\R^2$ and we plot it as a vector field on a regular grid of $20\times 20$ points contained in $(-2,2)^2$. This is done in Figure \ref{big}(c), where the vectors are automatically scaled to improve legibility, and in Figure \ref{big}(e), where the vectors are to scale.

The input of the algorithm is a cloud of $N=10^6$ sample points $\{(x,F(x)):x\in X\}$, where the set $X\subseteq (-2,2)^2$ is generated randomly. Such a cloud is visualized in Figure \ref{big}(b), except that only $10^3$ points are plotted in the figure to improve legibility.

The resulting Jacobian matrix estimator $\widehat J$ is visualized as a vector field in Figure \ref{big}(d) automatically scaled, and in Figure \ref{big}(f) to scale.

\begin{figure}[!ht]
    \centering
    \includegraphics[width=0.8\columnwidth, keepaspectratio]{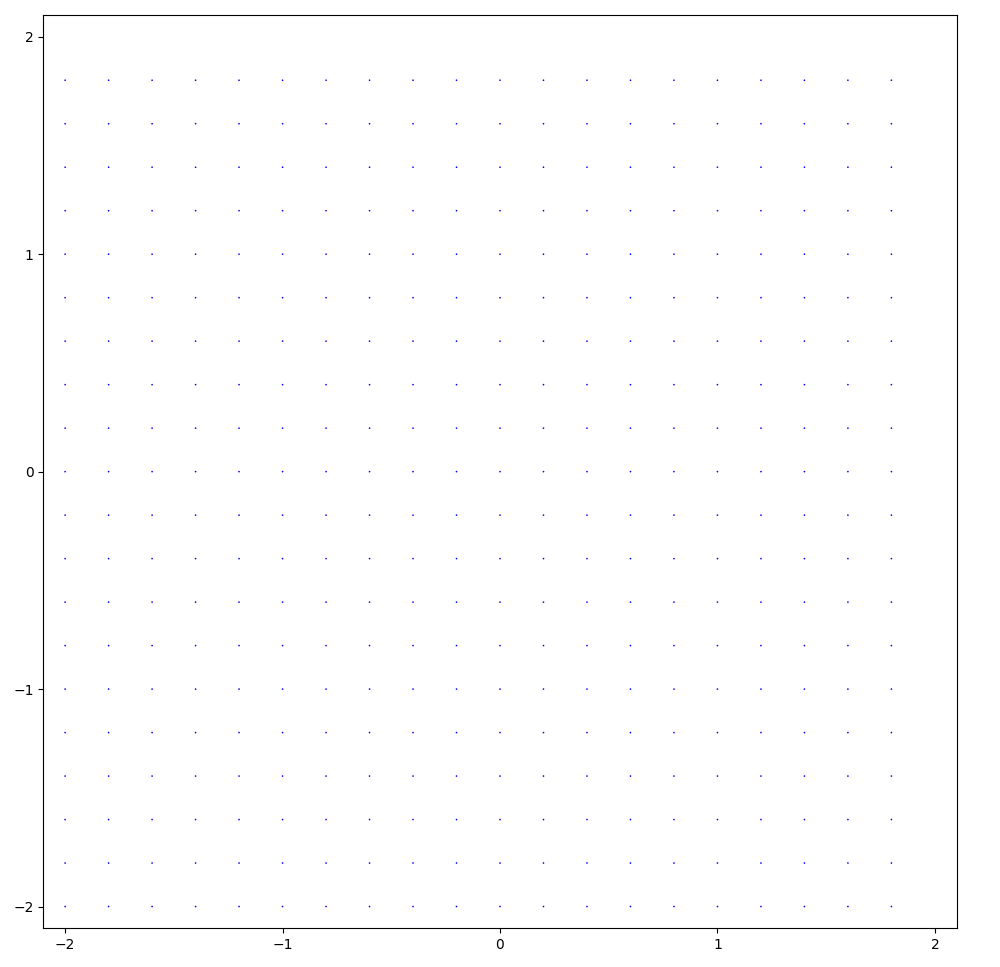}
    \caption{The difference $\widehat{J}-J$ (to scale).}\label{difference-grad}
\end{figure}

In order to compare the near-identical vector fields of $J$ and $\widehat J$, we offer Figure \ref{difference-grad} in which we plot the vector field $\widehat J-J$ to scale.

\subsection{Error statistics}\label{Ss:All}

It can be seen plainly from Figure \ref{difference-grad} that $\widehat J$ is a good estimator of $J$ for the function $F_0$. In order to quantify this fact for $F_0$ and for the other testing functions $F:\R^d\to\R^c$ from Table \ref{Tb:Functions}, we consider two validation methods.

For the first validation method, suppose that the Jacobian matrix estimator is trained for a function $F:\R^d\to\R^c$ that is not known to the algorithm but that is known to us. We can therefore compare the estimated Jacobian matrix $\widehat J$ with the actual Jacobian matrix $J$ of $F$ calculated by standard symbolic differentiation.

In more detail, let $S$ be a very large, randomly generated subset of the domain of $F$. (We used $|S|=10^6$ throughout.) For $\delta\ge 0$, let
\begin{displaymath}
    S_\delta =\{x\in S:\norm{J(x)}{\R^{cd}}>\delta\},
\end{displaymath}
where $\norm{J(x)}{\R^{cd}}$ is the Frobenius norm of the matrix $J(x)$, that is, the Euclidean norm of the vector in $\R^{cd}$ obtained by concatenating the rows of $J(x)$. Consider the error estimate
\begin{equation}\label{Eq:ErrorEstimate}
    E_\delta = \frac{1}{|S_\delta|}\sum_{x\in S_\delta}\frac{\norm{\widehat{J}(x)-J(x)}{\R^{cd}}}{\norm{J(x)}{\R^{cd}}}
\end{equation}
that averages pointwise over $S_\delta$ the size of the error $\widehat J(x)-J(x)$ relative to the size of $J(x)$. We will refer to $E_\delta$ as the \emph{average relative error} of $\widehat J$ over $S$ and report it in percents.

\begin{remark}
The average relative error $E_\delta$ is sensitive to the parameter $\delta$. By setting $\delta=0$, we exclude only those points $x\in S$ with $J(x)=0$, that is, the points of $S$ where the relative error is not defined. By setting $\delta$ to a small positive value, we also exclude the points $x\in S$ with small $J(x)$, hence in general improving the error estimate by ignoring the points where the relative error is greatly magnified by the small denominator.
\end{remark}

For the second validation method, suppose that we would like to train the Jacobian matrix estimator on a sample set $\{(x,F(x)):x\in X\}$ for a function $F:\R^d\to\R^c$ that is not known to us. We split $X$ into two disjoint subsets $X_T$ and $X_V$, a larger training set $X_T$ and a smaller validation set $X_V$. (We used $|X_V|=10^4$ except as otherwise noted.) We then train $\widehat J$ on the training set $\{(x,F(x)):x\in X_T\}$ and we calculate the error on the validation set $X_V$ as follows.

Let $X^*_V$ be the set of all pairs $(a,b)\in X_V\times X_V$ such that $b$ is a near neighbor of $a$ obtained by the same near neighbor routine that has been employed in Algorithm \ref{J-alg}. For $\delta>0$, let
\begin{displaymath}
    S^*_\delta = \{(a,b)\in X^*_V:\norm{F(b)}{\R^c}>\delta\}
\end{displaymath}
and note that $F(b)$ is known from the sample set. Instead of comparing $\widehat J$ to the unknown Jacobian matrix $J$, consider the error estimate
\begin{displaymath}
    E^*_\delta = \frac{1}{|S^*_\delta|} \sum_{(a,b)\in S^*_\delta}\frac{\norm{F(b)-(F(a)+\widehat{J}(a)(b-a))}{\R^c}}{\norm{F(b)}{\R^c}}
\end{displaymath}
based on linear approximations. The purpose of the parameter $\delta$ is the same as in the first error estimate \eqref{Eq:ErrorEstimate}. We again report $E^*_\delta$ in percents.

Table \ref{Tb:Results} summarizes the error estimates for all testing functions from Table \ref{Tb:Functions}.

\begin{table}[!ht]
\centering
\caption{The average relative error (in percents) of the Jacobian matrix estimator trained by Algorithm \ref{J-alg} for the functions from Table \ref{Tb:Functions}.}\label{Tb:Results}
\begin{tabular}{llrrrrr}
\toprule
function    & $N$        &$E_0$    & $E_{0.001}$ & $E_{0.01}$   &$E_{0.1}$ & $E^*_{0.01}$ \\
  \midrule
  $F_0$       &$10^3$   &$24.9$     & $24.9$  & $21.2$ & $11.4$ & $4.42$\\
            &$10^4$   &$5.78$      & $5.78$     &$5.22$ & $3.37$ & $2.73$ \\

            &$10^5$     &$3.11$      & $3.10$  & $2.83$  & $1.94$ & $2.69$ \\
            &$10^6$     &$1.20$      & $1.20$  & $1.11$  & $0.89$ & $2.69$  \\

  \midrule
  $F_1$       &$10^4$     & $0.57$  & $0.56$ & $0.56$    & $0.55$ & $0.70 $ \\
            &$10^5$     & $0.42$ & $0.42$  & $0.42$  & $0.40$ & $0.69$ \\
  \midrule
  $F_2$     & $10^4$ & $6.33$ & $2.72$ & $1.82$ & $1.33$ & $0.55$ \\
            & $10^5$ & $1.92$ & $1.21$ & $0.82$ & $0.54$ & $0.49$ \\
  \midrule
  $F_3$     & $10^4$ & $21.9$ & $7.25$ & $3.66$ & $1.65$ & $0.70$ \\
            & $10^5$ & $8.70$ & $1.65$ & $0.68$ & $0.30$ & $0.64$ \\
  \midrule
  $F_4$    & $10^4$ & $1.19$ & $1.19$ & $1.19$ & $1.19$ & $0.43$ \\
            & $10^5$ & $0.88$ & $0.88$ & $0.88$ & $0.88$ & $0.43$ \\
  \midrule
  $F_5$    & $10^4$ & $0.79$ & $0.79$  & $0.79$ & $0.79$ & $0.29$\\
        & $10^5$ & $0.24$ & $0.24$ & $0.24$ & $0.24$ &  $0.28$ \\
  \midrule
  $F_6$    & $10^4$ & $1.00$ & $1.00$  & $1.00$ & $1.00$ & $0.04$ \\
            & $10^5$ & $0.27$ & $0.27$ &$0.27$ & $0.27$ & $0.04$ \\
  \midrule
  $F_7$ &  $10^4$ & $6.84$ & $6.84$ & $6.84$ & $6.11$ & $1.44$ \\
            & $10^5$ & $2.92$ & $2.92$ & $2.92$ & $2.60$ & $1.40$  \\
  \midrule
  $F_8$ & $10^4$ & $4.61$ & $4.61$ & $4.60$ & $4.60$ & $3.61$ \\
            & $10^5$ & $1.90$ & $1.90$ & $1.89$ & $1.89$ & $3.60$ \\
  \midrule
  $F_9$ &$10^4$ & $1.99$ & $1.99$ & $1.99$ & $1.99$ & $0.28$ \\
            & $10^5$ & $0.95$ & $0.95$ & $0.95$ & $0.95$ & $0.28$ \\
  \midrule
  $F_{10}$ & $10^4$ & $7.75$ & $7.75$ & $7.74$ & $7.74$ & $3.91$ \\
            & $10^5$ & $3.81$ & $3.81$ & $3.81$ & $3.81$ & $0.43$ \\
            & $10^6$ & $1.99$ & $1.99$ & $1.99$ & $1.99$ & $0.43$\\
  \midrule
  $F_{11}$   & $10^4$ & $8.56$ & $8.56$ & $8.56$ & $8.56$ & $9.98$  \\
            & $10^5$ & $4.33$ & $4.33$ & $4.33$ & $4.33$ & $9.92$  \\
            & $10^6$ & $2.42$ & $2.42$ & $2.42$ & $2.42$ & $10.3$ \\
  \midrule
  $F_{12}$  & $10^5$ & $7.72$ & $7.72$ & $7.72$ & $7.52$ & $8.87$ \\
            &$10^6$ & $4.92$ & $4.92$ & $4.90$ & $4.87$ &  $8.47$\\
  \bottomrule
\end{tabular}
\end{table}

The effect of $\delta$ is most visible for functions whose Jacobian matrix is frequently close to $0$. Not surprisingly, the error estimate $E^*_\delta$ tends to be better than $E_\delta$ since both the training of $\widehat J$ and the error estimate $E^*_\delta$ are based on linear approximations, albeit at different pairs of points. Note that the error gets worse as the dimension $d$ of the domain increases.

The rather large error $E^*_{0.01}$ for $F_{11}$ and $F_{12}$ is a consequence of the large volume of the domains and small validation set $X_V$. Increasing the size of $X_V$ from $10^4$ (the default) to $10^5$ for $F_{11}$ (resp.~$F_{12}$) with $N=10^6$ improves $E^*_{0.01}$ from $10.3$ (resp.~$8.47$) to $3.92$ (resp. $3.53$).

\subsection{Parameters and limitations of the algorithm}\label{Ss:Notes}

We conclude this section with somewhat informal comments on the effects of varying the parameters and on the limitations of Algorithm \ref{J-alg}. A more thorough discussion will be reported elsewhere.

\subsubsection{Varying $k_{max}$ and $r_{max}$}

Generally speaking, the error improves with smaller values of $r_{max}$, provided that the training set is sufficiently dense so that enough neighbors can be found within radius $r_{max}$ of sample points. The effect of the parameter $k_{max}$ is more delicate. If the sample set is very dense, larger values of $k_{max}$ improve the error since more quality training pairs become available. If the sample set is sparse, larger values of $k_{max}$ (in conjunction with large values of $r_{max}$) make the error worse.

Table \ref{Tb:kr} lists observed error rates for the function $F_0$ with a sample set of size $N=10^4$ for various values of $k_{max}$ and $r_{max}$. The first line with $k_{max}=30$ and $r_{max}=0.5$ is repeated from Table \ref{Tb:Results} as a baseline; it turns out that in this particular example the pairs of nearby points are more restricted by $k_{max}=30$ than by $r_{max}=0.5$. In the second line, when $k_{max}$ is relaxed to $100$, the algorithm picks up too many distant pairs of points and the error rate gets worse. In the third line, with $k_{max}=100$ and $r_{max}=0.1$, the number of kept nearby pairs of points drops by approximately $80\%$ compared to the second line, only about $19$ nearby points are retained on average for every sample point (well below $k_{max}=100$), and the error rate improves slightly on the baseline. Finally, with $k_{max}=10$ and $r_{max}=0.5$, the number of kept training pairs drops by approximately $90\%$ compared with the second line and the error rate improves further.

\begin{table}[!ht]
\centering
\caption{The effect of $k_{max}$ and $r_{max}$ for $F_0$ and $N=10^4$.}\label{Tb:kr}
\begin{tabular}{rrrrrrr}
\toprule
$k_{max}$   & $r_{max}$     & $E_0$     & $E_{0.001}$   & $E_{0.01}$    &$E_{0.1}$  & $E^*_{0.01}$ \\
\midrule
$30$        & $0.5$         &$5.78$     & $5.78$        & $5.22$        & $3.37$    & $2.73$\\
$100$       & $0.5$         &$10.4$     &$9.54$         & $8.77$        &$5.35$     & $5.52$\\
$100$       & $0.1$         &$5.30$     &$5.30$         & $4.73$        &$2.92$     & $1.14$\\
$10$        & $0.5$         &$4.97$     &$4.96$         & $4.45$        &$2.92$     & $0.61$\\
\bottomrule
\end{tabular}
\end{table}

\subsubsection{Varying the geometry of the neural network}

We do not understand well the effect of the geometry of the neural network on Algorithm \ref{J-alg}. Table \ref{Tb:Geom} reports the error rate $E_{0.1}$ for the function $F_0$, sample set of size $N=10^4$ and the default parameters $k_{max}=30$ and $r_{max}=0.5$. The first line of Table \ref{Tb:Geom} is again taken from Table \ref{Tb:Results} as a baseline.

\begin{table}[!ht]
\centering
\caption{The effect of neural net geometry for $F_0$ and $N=10^4$.}\label{Tb:Geom}
\begin{tabular}{lr}
\toprule
layers and their size       & $E_{0.1}$ \\
\midrule
100, 100, 50, 20            &$3.37$\\
1000, 100, 50, 20           &$4.48$\\
1000, 1000, 50, 20          &$3.62$\\
100, 1000, 100              &$2.83$\\
50, 50                      &$2.23$\\
8 layers of 20 neurons      &$3.27$\\
16 layers of 20 neurons     &$5.50$\\
\bottomrule
\end{tabular}
\end{table}

\subsubsection{Non-differentiable functions and singularities}

Although we have assumed throughout that the function $F$ is differentiable for the purposes of being able to compare its Jacobian matrix $J$ with the Jacobian estimator $\widehat J$, Algorithm \ref{J-alg} will happily train $\widehat J$ from a sample set $\{(x,F(x)):x\in X\}$ for any function $F$, differentiable or not. In fact, $F$ need not be even defined outside of $X$.

It is to be expected that $\widehat J$ will not be a good approximation of $J$ if $J$ oscillates wildly relative to the density of the sample set or if it attains a very large range of values, for example due to the presence of a singularity inside or just outside of the considered domain.

Concerning singularities, we observed the following results for $\widehat J$ trained on a sample set with $N=10^4$ points. For $F(x,y) = (x+y)^{1/2}$ on $(0,1)^2$, the Jacobian matrix has a singularity at $(0,0)$, that is, at a ``corner'' of the domain of $F$, nevertheless the error estimates are very satisfactory: $E_0 = 0.94$ and $E^*_{0.01}=0.02$. For $F(x,y)=\sqrt{x^2+y^2}$ on $(-1,1)^2$, there is a $J$-singularity at $(0,0)$, that is, inside the domain of $F$, and the error estimates are $E_{0.1}=23.6$ and $E^*_{0.01}=0.3$. Finally, the function $F(x,y)=|x|+|y|$ is not differentiable along the coordinate axes, but on $(-1,1)^2$ we still get $E_0=3.9$ and $E^*_{0.01}=0.4$.

The low values of $E^*_\delta$ indicate that $\widehat J$ does not seem to be capable of detecting/suggesting singularities for truly unknown functions (for which the $E_\delta$ error estimates are not available).

\subsubsection{Convexity}

If the function $F$ is scalar valued, then convexity of the function $F$ has an effect on the error estimate. If the function $F$ is convex (resp.~concave) near $a$, the estimator $\widehat J$ tends to return a value $\widehat J(a)$ that is smaller (resp.~larger) in norm than it should be. To see this, consider the concave function depicted in Figure \ref{Fig1D} and suppose that $\widehat J(a)$ is such that the green secant line intersects the line $x=b$ somewhere between the red tangent line and the blue graph of $F$. Then the value $\widehat J(a)$ should be increased in order to get closer to the tangent line, but the still positive loss function will have the opposite effect.

\subsubsection{Noisy sampling data}

Algorithm \ref{J-alg} may be used on noisy training data. By training $\widehat J$ on the sample set $\{(x,F(x)+\epsilon_x):x\in X\}$, where $\{\epsilon_x:x\in X\}$ is a family of independent Gaussian random variables with mean $0$ and standard deviation $0.01$, we observed that $\widehat J$ still approximates $J$ rather well, but certainly not as well as in noiseless situations. For instance, for the function $F_1$ of Table \ref{Tb:Functions} and with $N=10^5$ sample points we observed the average relative error of $E_{0.01}=8.06$ percent, for $F_4$ and $N=10^4$ we observed $E_{0.01}=9.74$, and for $F_8$ and $N=10^5$ we observed $E_{0.01}=10.5$. The negative effect of noise on regression in general was discussed in the introduction and it persists in the context of Algorithm \ref{J-alg}.

\subsubsection{The running time}

The running time of the algorithm increases with the size $N$ of the sample set and with the parameter $k_{max}$. Using $k_{max}=50$, the observed running time of Algorithm \ref{J-alg} on a PC with Intel Core i7 9th generation 3GHz processor was $0.6$ sec/epoch for $N=10^3$, $4$ sec/epoch for $N=10^4$, $40$ sec/epoch for $N=10^5$ and $250$ sec/epoch for $N=10^6$.

\section{Convergence of the Jacobian matrix estimator: A formal proof}\label{Sc:Math}

In this section we prove that under reasonable assumptions on the function $F$, the neural network, the training set and the outcome of the training in Algorihtm \ref{J-alg}, the resulting Jacobian matrix estimator $\widehat{J}$ converges in norm to the Jacobian matrix $J$ of $F$.

We will be more careful with vectors from now on and write them as column vectors. As above, the Euclidean norm of $x=(x_1,\dots,x_d)^\top\in\R^d$ will be denoted by
\begin{displaymath}
    \norm{x}{\R^d} = \left(\sum_{i=1}^d x_i^2\right)^{1/2}.
\end{displaymath}
The dot product of $x$, $y\in \R^d$ will be denoted by
\begin{displaymath}
    \langle x,y\rangle = \sum_{i=1}^d x_iy_i,
\end{displaymath}
so that $\norm{x}{\R^d} = \langle x,x\rangle^{1/2}$. Finally, for an $n\times m$ matrix $A$, the operator norm will be denoted by
\begin{displaymath}
    \norm{A}{n\times m} = \mathrm{sup}\{\norm{Ax}{\R_n}:x\in \R^m,\,\norm{x}{\R^m}\le 1\}.
\end{displaymath}

\subsection{The assumptions}

Let us fix $\varepsilon>0$ throughout. The first assumption states that the second partial derivatives of $F$ are bounded:

\begin{assumption}\label{fn-ass}
Let $U$ be an open bounded subset of $\R^d$. Let $F=(F_1,\dots,F_c):U\to\R^c$ be a twice differentiable function and let $H=(H_i)_{i=1}^c$,
\begin{displaymath}
    H_i=\left(\frac{\partial^2 F_i}{\partial x_j\partial x_k}\right)_{j,k}.
\end{displaymath}
There is a constant $L>0$ such that $\norm{H_i}{d\times d}\le L$ for every $1\le i\le c$.
\end{assumption}

The second assumption states that the neural network trained by Algorithm \ref{J-alg} is Lipschitz. This can be achieved by fixing the geometry of the neural network (number of layers and neurons in each layer), by using activation functions that are Lipschitz, and by bounding above the weights of the neural network, for instance.

\begin{assumption}\label{net-ass}
There is a constant $L'>0$ such that every neural network $\widehat J$ trained by Algorithm \ref{J-alg} satisfies
\begin{displaymath}
    \norm{\widehat{J}(y)-\widehat{J}(x)}{c\times d}\le L'\norm{y-x}{\R^d}
\end{displaymath}
for every $x,y\in U$.
\end{assumption}

The third assumption states, roughly speaking, that the (domain of) the sample set is sufficiently dense and that near every sample point we can find $d$ additional sample points such that any two of the $d$ resulting vectors are close to being orthogonal.

\begin{assumption}\label{set-ass}
Let $\{(x,F(x)):x\in X\}$ be the sample set. Then:
\begin{enumerate}
\item[(i)] $X$ is $\varepsilon$-dense in $U$, that is, for every $y\in U$ there is $x\in X$ such that $\norm{x-y}{\R^d}\le\varepsilon$,
\item[(ii)] there exist a constant $0<\alpha<1$ and a constant $R>0$ such that for all $x \in X$, if $U_x$ is the set of the $k_{max}$ nearest neighbors of $x$ within distance $R\varepsilon$ of $x$, then there are points $u_{x,1},\dots,u_{x,d} \in U_x$ such that
\begin{displaymath}
    \frac{\left| \inner{u_{x,i}-x}{u_{x,j}-x}{}  \right|}{\norm{u_{x,i}-x}{\R^d} \norm{u_{x,j}-x}{\R^d}} \leq \frac{\alpha}{d}
\end{displaymath}
for every $1\le i<j\le d$.
\end{enumerate}
\end{assumption}

\begin{remark}
The constant $R$ in Assumption \ref{set-ass} can be shown to always exist, for any $\alpha\in(0,1)$ and $d$, as long as $U$ has large enough diameter and $k_{max}=|X|$. However, in general, $k_{max}$ is much smaller than $|X|$, so Assumption \ref{set-ass}(ii) is not implied by (i) in general.
\end{remark}

The final assumption states that Algorithm \ref{J-alg} produces a neural network for which the loss (used in training) is under control. This should be seen as a relative assumption on the capability of neural networks.

\begin{assumption}\label{training-ass}
After training $\widehat J$ by Algorithm \ref{J-alg} with $r_{max}\geq R\varepsilon$, we have
\begin{displaymath}
    \frac{\norm{F(y)-F(x)-\widehat{J}(x)(y-x)}{\R^c}}{\norm{y-x}{\R^d}}\le\varepsilon
\end{displaymath}
for every $x\in X$ and $y\in U_x$, where $X$ and $U_x$ are as in Assumption \ref{set-ass}.
\end{assumption}

\subsection{A proof of convergence}

We will need the following result in the proof of the main theorem.

\begin{lemma}\label{Lm:Coeffs}
Let $0<\alpha<1$ and let $B=\{b_1,\dots,b_d\}$ be a set of unit vectors in $\R^d$ such that $|\langle b_i,b_j\rangle|\le \alpha/d$ for every $1\le i<j\le d$. Then $B$ is a basis of $\R^d$ and if $y=\sum_{i=1}^d y_ib_i$ is any vector with $\norm{y}{\R_d}\le 1$ then
\begin{equation}\label{Eq:CoeffsBounded}
    |y_i|\le (1-\alpha)^{-1/2}
\end{equation}
for every $1\le i\le d$.
\end{lemma}
\begin{proof}
Let us view $B$ as a matrix with columns $b_1,\dots,b_d$ and let $G=B^\top B$ be the (symmetric) Gram matrix of $B$. Since all the vectors $b_i$ are of unit length and $|\langle b_i,b_j\rangle|\le\alpha/d$ for every $1\le i<j\le d$, we have $G = I + M$ for some matrix $M$ such that the absolute value of every entry of $M$ is at most $\alpha/d$. Hence $\norm{M}{d\times d} \le \alpha<1$, $G$ is invertible and thus also $B$ is invertible.

Let $y$ be a vector in $\R^d$ such that $\norm{y}{\R^d}\le 1$. Since $B$ is a basis, we can write $y=\sum_{i=1}^d y_ib_i$, i.e., $y = B(y_1,\dots,y_d)^\top$.

Recall the Neumann series \cite[VII, Corollary 2.3]{Conway}
\begin{displaymath}
    (I-T)^{-1} = \sum_{k=1}^\infty T^k
\end{displaymath}
for a bounded linear operator $T$. Since $G=I+M=I-(-M)$ and $\norm{M}{d\times d}=\norm{-M}{d\times d}$, we deduce
\begin{displaymath}
    \norm{G^{-1}}{d\times d} \leq \sum_{k=0}^\infty (\norm{M}{d\times d})^k \le\sum_{k=0}^\infty\alpha^k = (1-\alpha)^{-1}.
\end{displaymath}
Then
\begin{equation}\label{Eq:Aux}
    \norm{G^{-1/2}}{d\times d}\le (1-\alpha)^{-1/2}
\end{equation}
by the spectral mapping theorem \cite{Conway}, as $G$ is positive. (This can also be seen by noting that $G^{-1}$ is symmetric, so its operator norm is the absolute value of its largest eigenvalue, whose square root is then the largest eigenvalue, hence the norm, of the symmetric matrix $G^{-1/2}$.)

Let $E = B G^{-1/2}$ and note that we have $E^\top E = (BG^{-1/2})^\top BG^{-1/2} = G^{-1/2}B^\top B G^{-1/2} = G^{-1/2}GG^{-1/2}= I$. Hence the columns $e_1,\ldots,e_d$ of $E$ form an orthonormal basis for $\R^d$.

Consider the vector
\begin{displaymath}
    z = (\inner{e_1}{y}{},\dots,\inner{e_d}{y}{})^\top.
\end{displaymath}
Since $E$ is an orthonormal basis, we have $\norm{z}{\R^d} = \norm{y}{\R^d}\le 1$ and $y = Ez = BG^{-1/2}z$. Recalling $y=B(y_1,\dots,y_d)^\top$, we deduce $(y_1,\dots,y_d)^\top = G^{-1/2}z$. Therefore
\begin{displaymath}
    \norm{(y_1,\dots,y_d)}{\R^d}\le \norm{G^{-1/2}}{d\times d}\norm{z}{\R^d}\le (1-\alpha)^{-1/2}
\end{displaymath}
thanks to \eqref{Eq:Aux} and $\norm{z}{\R_d}\le 1$. This means that
\begin{displaymath}
    \sum_{i=1}^d y_i^2\le (1-\alpha)^{-1}
\end{displaymath}
and \eqref{Eq:CoeffsBounded} follows.
\end{proof}

We are now ready to state an prove the main result.

\begin{theorem}\label{Th:Main}
Let $\varepsilon>0$, $U\subseteq\R^d$, $F:U\to\R^c$ and let $J$ be the Jacobian matrix of $F$. Suppose that Assumptions \ref{fn-ass}--\ref{training-ass} are satisfied for $F$ and for the trained Jacobian matrix estimator $\widehat J$. Then
\begin{displaymath}
    \sup_{x\in U}\norm{\widehat{J}(x)-J(x)}{c\times d}\le C\varepsilon,
\end{displaymath}
where
\begin{displaymath}
    C=(L+L')+\frac{d}{(1-\alpha)^{1/2}}\left(1+\frac{L R}{2}\right).
\end{displaymath}
\end{theorem}
\begin{proof}
Since the Hessian of $F$ is bounded above by $L$ by Assumption \ref{fn-ass}, a higher order Taylor expansion for $F$ yields
\begin{displaymath}
  \norm{F(y)-F(x)-J(x)(y-x)}{\R^c} \leq \frac{L}{2} \norm{x-y}{\R^d}^2
\end{displaymath}
for all $x,y \in U$. Using this inequality, Assumption \ref{training-ass} and the triangular inequality, we have
\begin{align}
    &\norm{(\widehat{J}(x)-J(x))(y-x)}{\R^c}\notag \\
    &\quad\leq\norm{\widehat{J}(x)(y-x) - (F(y)-F(x))}{\R^c} \notag\\
    &\quad\quad+\norm{F(y)-F(x)-J(x)(y-x)}{\R^c} \label{Eq:Error} \\
    &\quad\leq\varepsilon\norm{y-x}{\R^d}+\frac{L}{2}\norm{y-x}{\R^d}^2.\notag
\end{align}
Assumption \ref{fn-ass} implies that $J$ is $L$-Lipschitz, i.e.,
\begin{displaymath}
  \norm{J(x)-J(y)}{c\times d} \leq L \norm{x-y}{\R^d}
\end{displaymath}
for every $x,y\in U$. By Assumption \ref{net-ass}, $\widehat{J}$ is $L'$-Lipschitz. Therefore
\begin{align}
    &\norm{(\widehat{J}(x)-J(x))y}{\R^c}\notag\\
    &\quad\leq\norm{(\widehat{J}(x_0)-J(x_0))y}{\R^c}\label{Eq:Error2}\\
    &\quad\quad+(L'+L)\norm{x_0-x}{\R^d}\norm{y}{\R^d}\notag
\end{align}
for every $x,y,x_0\in U$.

For the rest of the proof, fix $x\in U$ and let $y$ be any vector with $\norm{y}{\R^d}\le 1$. By Assumption \ref{training-ass}(i), there exists $x_0\in X$ such that $\norm{x-x_0}{\R^d}\le\varepsilon$. By Assumption \ref{training-ass}(ii), there exist points $u_i = u_{x_0,i}\in X$ within $R\varepsilon$ of $x_0$ such that the unit vectors
\begin{displaymath}
    b_i=\frac{u_i-x_0}{\norm{u_i-x_0}{\R^d}}
\end{displaymath}
satisfy $|\inner{b_i}{b_j}{}|\le \alpha/d$ for every $1\le i<j\le d$. By Lemma \ref{Lm:Coeffs}, $B=\{b_1,\dots,b_d\}$ is a basis of $\R^d$ and $y = \sum_{i=1}^d y_ib_i$ for some $y_1,\dots,y_d\in\mathbb R$ such that $|y_i|\le (1-\alpha)^{-1/2}$ for every $1\le i\le d$.

Using \eqref{Eq:Error} in the first inequality below,
\begin{align*}
    &\norm{(\widehat{J}(x_0)-J(x_0))b_i}{\R^c}\\
    &\quad=\frac{\norm{(\widehat{J}(x_0)-J(x_0))(u_i-x_0)}{\R^c}}{\norm{u_i-x_0}{\R^d}}\\
    &\quad\le\frac{\varepsilon\norm{u_i-x_0}{\R^d}+\frac{L}{2}\norm{u_i-x_0}{\R^d}^2}{\norm{u_i-x_0}{\R^d}}\\
    &\quad=\varepsilon + \frac{L}{2}\norm{u_i-x_0}{\R^c} \le \varepsilon + \frac{L R \varepsilon}{2} = \varepsilon\left(1 + \frac{L R}{2}\right).
\end{align*}
We therefore have
\begin{align*}
    &\norm{(\widehat{J}(x_0)-J(x_0))y}{\R^c}\\
    &\quad=\norm{(\widehat{J}(x_0)-J(x_0))\sum_{i=1}^d y_ib_i}{\R^c}\\
    &\quad\le\sum_{i=1}^d|y_i|\norm{(\widehat{J}(x_0)-J(x_0))b_i}{\R^c}\\
    &\quad\le\sum_{i=1}^d(1-\alpha)^{-1/2}\norm{(\widehat{J}(x_0)-J(x_0))b_i}{\R^c}\\
    &\quad\le \varepsilon \frac{d}{(1-\alpha)^{1/2}}\left(1 + \frac{L R}{2}\right).
\end{align*}
The inequality \eqref{Eq:Error2} then yields
\begin{align*}
    &\norm{(\widehat{J}(x)-J(x))y}{\R^c}\\
    &\quad\le\norm{(\widehat{J}(x_0)-J(x_0))y}{\R^c} +\varepsilon(L'+L)\\
    &\quad\le\varepsilon\left((L+L')+\frac{d}{(1-\alpha)^{1/2}}\left(1+\frac{L R}{2}\right)\right),
\end{align*}
finishing the proof.
\end{proof}

The estimate in Theorem \ref{Th:Main} does not depend on the dimension $c$ of the codomain of $F$, which reflects the fact that computing the gradient of each coordinate of $F$ does not affect the computation of the gradient of the other coordinates. Of course, the estimate gets worse as the dimension $d$ of the domain grows larger.

\section{Conclusion and future work}

We introduced a novel algorithm for the estimation of the Jacobian matrix of an unknown, sampled multivariable function by means of neural networks. The main ideas of the algorithm are a loss function based on a linear approximation and a nearest neighbor search in the sample data. The algorithm was tested on a variety of functions and for various sizes of sample sets. The typical average relative error is on the order of single percents, using both an error estimate for functions with a known Jacobian matrix and an error estimate for unknown functions based on linear approximations. We proved that the estimated Jacobian matrix converges to the Jacobian matrix under reasonable assumptions on the function, the sampling set and the loss function.

In future work, we will apply Algorithm \ref{J-alg} for validation of physics-informed models, anomaly detection and time series analysis. For physics-informed models, a typical restriction is of the form $\partial F/\partial x>0$, which can be verified or refuted by the Jacobian matrix estimator $\widehat J$ trained by Algorithm \ref{J-alg}. In anomaly detection, the Jacobian estimator can be retrained periodically on a window of data and deviations in the values of $\widehat J$ can be statistically detected (and, in addition, the reason for the anomaly can be narrowed down by focusing on anomalous values $\partial F_i/\partial x_j$). In a time series, the time parameter can be treated as another variable (preferably modulo a fixed period of time to allow for nearby points in the sample set) or the time parameter can be suppressed by turning the time series $S(t_1)=(a_1,\dots,a_n)$, $S(t_2)=(b_1,\dots,b_n)$, etc, into a dynamical system $F(a_1,\dots,a_n)=(b_1,\dots,b_n)$, etc.

\section*{Acknowledgement}

The authors acknowledge support from a Lockheed Martin Space Engineering and Technology grant ``Time series analysis.''


\newpage

\section*{Appendix: The code for JacobianEstimator}

\begin{lstlisting}[language=Python,basicstyle=\tiny, keywordstyle=\color{red}, numberstyle=\color{pink}, stringstyle=\color{purple}, identifierstyle=\color{blue} , commentstyle=\color{black}]

# Imported modules

import math
import numpy as np
from sklearn.neighbors import NearestNeighbors
import tensorflow as tf
from tensorflow import keras
from keras.models import Sequential
from keras.layers import Dense
from keras.optimizers import adam_v2
from tensorflow.keras.constraints import max_norm

################################
## Definition of the class  Jhat
##
## Purpose:
## Estimate the Jacobian of F : I -> O from a finite sample of values (x,F(x)).
## The function F is considered unknown outside of the sample set.
## I is a subset of R^domdim and O is a subset of R^codomdim.
##
## Usage:
## 1. Create a Jhat object. This process does not involve the data set,
##    but it does involve the dimensions of the data set. This process creates
##    a simple ANN using Dense layers.
## 2. Fit the Jhat object to the sample data. The data comes as an array of
##    input values (x) and a corresponding array of output values (fx).
##    Formally, x and fx have respective "shapes" (N,domdim) and (N,codomdim),
##    where N is the number of sample points. We understand df[j] as F(x[j]).
## 3. Use the functions predict1 (for one input value) or predict (for a list
##    of input values) to obtain the estimate for the Jacobian of F.
##
## Syntax:
##
## constructor:
## Jhat(layers,domdim,codomdim,nbr,r_max,batch_size,epochs,learning_rate,verbose)
##      - layers (opt): an array of integers with the number of neurons per layer.
##          The resulting ANN will have two extra layers (input and output).
##          Default value is [100,100,50,20].
##      - domdim (opt): dimension of the domain of F. Default is 2.
##      - codomdim (opt): dimension of the codomain of F. Default to 1.
##      - nbr (opt): max number of nearest neighbors to use for the estimation.
##          Default is twice domdim. Should be at least domdim.
##      - r_max: if not 0, only keeps neighbors within r_max. Default is 0.
##      - batch_size (opt): batch size for the gradient method. Default to 50.
##      - epochs (opt): number of iterations over training set. Default is 50.
##      - learning_rate (opt): learning rate for the gradient method.
##      - max_w (opt): maximum norm of the weights (no constraint if 0)
##      - verbose (opt): text notification during training. Default is True.
##
## fit: fit(self,x,fx)
##      - x: array of N input values, each an array of domdim numbers
##      - fx: array of N output values, each an array of codomdim numbers
## returns self
##
## predict1: predict1(self,x)
##      - x: an array of domdim numbers
## returns the predicted value of the Jacobian at x
##
## predict: predict(self,x)
##      - x: an array of N input values, each an array of domdim numbers
## returns array of predicted values of the Jacobian at each input value of x

class Jhat:

    # constructor
    def __init__(self, domdim = 2, codomdim = 2, layers = [100,100,50,20], batch_size = 50, epochs = 50, nbr = 0, r_max = 0.0, learning_rate = 0.0001, max_w=0,verbose=True):

        # store parameters
        self.domdim = domdim
        self.codomdim = codomdim
        self.layers = layers
        self.batch_size = batch_size if batch_size > 0 else 50
        self.epochs = epochs if epochs > 0 else 50
        self.nbr = nbr if nbr >= domdim else 2*domdim
        self.r_max = r_max
        self.max_norm = max_w
        self.verbose = verbose

        # create and store ANN
        self.model = Sequential()
        # input layer
        self.model.add(Dense(layers[0],input_shape=(domdim,),activation='swish',kernel_constraint=max_norm(max_w) if max_w>0 else None))
        # hidden layers
        for n_neurons in layers[1:]:
            self.model.add(Dense(n_neurons,activation='swish',kernel_constraint=max_norm(max_w) if max_w>0 else None))
        # output layer
        self.model.add(Dense(domdim * codomdim,kernel_constraint=max_norm(100.) if max_w>0 else None))
        self.model.compile(loss=self.createLoss(),optimizer=adam_v2.Adam(learning_rate=learning_rate))

    # internal: loss function generator
    # The core of the entire process, this produces a closure used as the loss
    # function to train the network from the data cloud to estimate the Jacobian
    def createLoss(self):
        def loss(real,predict):
            dx=real[:,0:self.domdim]
            df=real[:,self.domdim:]
            return tf.math.reduce_mean(tf.math.square(tf.math.subtract(df,tf.linalg.matvec(tf.reshape(predict,(self.batch_size,self.codomdim,self.domdim)),dx))))
        return loss

    # internal: prepare the data set for training.
    # For a given pair (x,F(x)) in the submitted sample:
    # 1. find self.nbr closest neighbors x_1,..,x_nbr to x.
    # 2. compute and store (x-x_1, F(x)-F(x_1)), (x-x_2,F(x)-F(x_2)), ...
    # 3. return the array obtained by processing each sample point.
    def prepareData(self,x,fx,train_mode=True,zero=0):
        if self.verbose :
            print("Preparing data from sample")
            print("Input shape ",x.shape)
            print("Output shape ",fx.shape)

        nbrs = NearestNeighbors(n_neighbors=self.nbr,algorithm='ball_tree').fit(x)
        dist, indices = nbrs.kneighbors(x)

        if self.verbose:
            print("Minimal distance :",np.amin(dist))
            print("Average distance :",np.average(dist))
            print("Maximal distance :",np.amax(dist))

        if self.r_max==0:
            self.r_max = np.amax(dist)+1.0

        N = x.shape[0] * self.nbr
        pos = np.empty( (N,self.domdim) )
        delta = np.empty( (N,self.domdim+self.codomdim) )
        reservedx = np.empty( (N,self.domdim) )
        reservefx = np.empty( (N,self.codomdim) )
        k = 0

        if train_mode:
            for idx in indices[:,0]:
                for idx2 in indices[idx,1:]:
                    dx = np.subtract(x[idx2],x[idx])
                    dxnorm = np.linalg.norm(dx)
                    if dxnorm < self.r_max and dxnorm>zero:
                        pos[k] = x[idx]
                        delta[k,0:self.domdim] = dx/dxnorm
                        delta[k,self.domdim:] = np.subtract(fx[idx2],fx[idx])/dxnorm
                        k = k + 1
        else:
            for idx in indices[:,0]:
                for idx2 in indices[idx,1:]:
                    dx = np.subtract(x[idx2],x[idx])
                    dxnorm = np.linalg.norm(dx)
                    div = np.linalg.norm(fx[idx2])
                    if dxnorm < self.r_max and dxnorm>0 and div>zero:
                        pos[k] = x[idx]
                        delta[k,0:self.domdim] = dx/div
                        delta[k,self.domdim:] = np.subtract(fx[idx2],fx[idx])/div
                        k = k + 1

        if k < self.batch_size:
            self.batch_size = k
        else:
            r = k % self.batch_size
            if r != 0:
                d = self.batch_size - r
                for j in range(d):
                    pos[k+j] = pos[0]
                    delta[k+j] = delta[0]
                k += d

        if self.verbose:
            print("Number of training data points: ", k)
            print("Finalized batch size: ", self.batch_size)

        # shuffle
        randind = np.arange(k)
        np.random.shuffle(randind)
        pos = pos[:k]
        delta = delta[:k]
        pos = [ pos[j] for j in randind ]
        delta = [ delta[j] for j in randind ]

        return (tf.convert_to_tensor(pos[:k]), tf.convert_to_tensor(delta[:k]))

    # fit: train the ANN with a sample set, passed as a parameter.
    def fit(self,x,fx):
        F = self.prepareData(x,fx)
        self.model.fit(F[0],F[1], epochs=self.epochs, batch_size=self.batch_size,verbose=2 if self.verbose else 0)
        return self


    # predict: compute the estimate of the Jacobian of F at each entry of x.
    def predict(self,x):
        N = x.shape[0]
        return self.model.predict(x).reshape(N,self.codomdim,self.domdim)

    # predict1: compute the estimate of the Jacobian of F at a single input x.
    def predict1(self,x):
        return self.model.predict([x])

    # predictflat: same as predict, but does not reshape output
    def predictflat(self,x):
        return self.model.predict(x)

    # validation functions
    def tangentl(self,j,dx,N):
        return tf.linalg.matvec(tf.reshape(j,(N,self.codomdim,self.domdim)),dx)

    def validate(self,x,fx,zero=0):
        F = self.prepareData(x,fx,train_mode=False,zero=zero)
        dfhat = tf.convert_to_tensor(self.predict(F[0]))
        dx=tf.cast(F[1][:,0:self.domdim],dtype=np.float32)
        df=tf.cast(F[1][:,self.domdim:],dtype=np.float32)
        N=df.shape[0]
        return tf.subtract(df,self.tangentl(dfhat,dx,N))
\end{lstlisting}

\end{document}